\documentclass[sigconf]{acmart}
\usepackage{graphicx}
\usepackage{amsmath}
\usepackage{amsthm}
\usepackage{amsfonts}
\usepackage{bm}
\usepackage{caption}
\usepackage{multirow}
\usepackage[english]{babel}
\usepackage{adjustbox}
\usepackage{subfigure}
\usepackage{booktabs}
\usepackage{array}
\usepackage{algorithm}
\usepackage{algorithmic}

\theoremstyle{definition}

\newtheorem{theorem}{Theorem}
\newtheorem{lem}{Lemma}
\newtheorem{corollary}{Corollary}

\AtBeginDocument{%
  \providecommand\BibTeX{{%
    \normalfont B\kern-0.5em{\scshape i\kern-0.25em b}\kern-0.8em\TeX}}}

\copyrightyear{2021}
\acmYear{2021}
\acmConference[WWW '21]{Proceedings of the Web Conference 2021}{April 19--23, 2021}{Ljubljana, Slovenia}
\acmBooktitle{Proceedings of the Web Conference 2021 (WWW '21), April 19--23, 2021, Ljubljana, Slovenia}
\acmPrice{}
\acmDOI{10.1145/3442381.3449951}
\acmISBN{978-1-4503-8312-7/21/04}
\begin{document}
\fancyhead{}
\title{Theoretically Improving Graph Neural Networks via Anonymous Walk Graph Kernels}

\author{Qingqing Long}
\authornote{These authors contributed equally to the work.}
\affiliation{
  \institution{Key Laboratory of Machine Perception (Ministry of Education), Peking University, China}
}
 \email{qingqinglong@pku.edu.cn}
 
 \author{Yilun Jin}
 \authornotemark[1]
\affiliation{%
  \institution{The Hong Kong University of Science and Technology}
  \city{Hong Kong SAR, China}
  }
  \email{yilun.jin@connect.ust.hk}

\author{Yi Wu}
 \authornotemark[1]
\affiliation{
  \institution{Peking University, China}
}
 \email{reyn19990619@pku.edu.cn}

\author{Guojie Song}
\authornote{Corresponding Author.}
\affiliation{%
  \institution{Key Laboratory of Machine Perception (Ministry of Education), Peking University, China}
}
\email{gjsong@pku.edu.cn}

\begin{abstract}
Graph neural networks (GNNs) have achieved tremendous success in graph mining. However, the inability of GNNs to model substructures in graphs remains a significant drawback. Specifically, message-passing GNNs (MPGNNs), as the prevailing type of GNNs, have been theoretically shown unable to distinguish, detect or count many graph substructures. While efforts have been paid to complement the inability, existing works either rely on pre-defined substructure sets, thus being less flexible, or are lacking in theoretical insights. In this paper, we propose GSKN\footnote{Our codes are available at  https://github.com/YimiAChack/GSKN}, a GNN model with a theoretically stronger ability to distinguish graph structures. Specifically, we design GSKN based on anonymous walks (AWs), flexible substructure units, and derive it upon feature mappings of graph kernels (GKs). We theoretically show that GSKN provably extends the 1-WL test, and hence the maximally powerful MPGNNs from both graph-level and node-level viewpoints. Correspondingly, various experiments are leveraged to evaluate GSKN, where GSKN outperforms a wide range of baselines, 
% on both synthetic and real-world datasets
endorsing the analysis. 
\end{abstract}



\maketitle

\section{Introduction}
Graph Neural Networks (GNNs) \cite{wu2020comprehensive,xu2018powerful} have achieved tremendous success in mining information networks, such as social networks, knowledge graphs, and biochemical networks \cite{duvenaud2015convolutional,niepert2016learning,kipf2016semi,hamilton2017inductive}. Underlying most widely studied GNNs \cite{velivckovic2017graph,kipf2016semi,hamilton2017inductive} lies the mechanism of \textit{message passing} \cite{gilmer2017neural,bayati2011dynamics}, or \textit{neighborhood aggregation} \cite{xu2018powerful}, where the representations for each node are updated by aggregating information from their neighbors. It is such a mechanism that enables efficient recursive computation of graph and node representations and promotes its tremendous success. 

However, message passing and the resultant message passing GNNs (MPGNNs), as the majority of GNNs, bear inherent drawbacks in capturing substructures of graphs. For example, \cite{xu2018powerful,morris2019weisfeiler} show that the capability of distinguishing different graphs of MPGNNs is upper-bounded by the 1-Weisfeiler-Lehman Isomorphism Test (1-WL Test), which is known unable to distinguish many common substructures (examples in \cite{sato2020survey}). In addition, \cite{chen2020can,arvind2020weisfeiler} show that MPGNNs are unable to count or detect substructures with 3 or more nodes. Both theoretical results uncover crucial limitations of MPGNNs, as graph substructures are widely recognized as indicative in various complex networks. For example, cliques, or community structures are common patterns in social networks \cite{granovetter1973strength}, and rings serve as indicators on functional groups in molecular chemistry \cite{liu2017surveying}, etc. Whether and how we can extend MPGNNs to account for substructures in complex graphs is thus an important open problem, posed by the contrast between the significance of substructures, and the inability of MPGNNs to model them. 

There have been several attempts aiming to improve MPGNNs for better modeling of substructures \cite{bouritsas2020improving,lee19-motif-attention,long2020graph,JinSS20}. However, \cite{lee19-motif-attention,bouritsas2020improving} design their solutions based on pre-defined substructure patterns, namely motifs and graphlets, which should be fixed prior to model learning. The reliance on fixed sets of substructures leads to a lack of versatility when applied to diverse real-world networks, each of which may be characterized by different substructures. On the other hand, \cite{JinSS20,long2019hierarchical} eliminate the reliance on pre-selected substructures through a flexible substructure pattern named \textit{anonymous walks} (AW) \cite{micali2016reconstructing}. However, their solutions are largely empirical, in that they fail to compare their solutions against MPGNNs in a principled theoretical manner. As we would like to preserve the versatility and independence from pre-selection, studying AW based GNNs in a theoretically principled manner becomes our primary focus. 

We resort to a parallel line of work on graph mining, namely \textit{Graph Kernels} (GKs). Graph kernels are classical methods that measure similarities between pairwise substructures, nodes, or graphs and hence enable graph clustering, comparison, and classification \cite{narayanan2017graph2vec,she2011weisfeiler,zhang2018retgk}. We consider graph kernels appropriate for our problem for three reasons. First, graph kernels inherently involve comparisons between substructures, such as subtrees \cite{she2011weisfeiler}, graphlets \cite{shervashidze2009efficient} and random walks \cite{kang2012fast}, making it natural for us to incorporate and analyze AW in the framework. In addition, connections between GKs and GNNs have been well-identified \cite{lei2017deriving,chen2020convolutional,du2019graph}, making it easier to adapt the theory on graph kernels to design and analyze GNNs. Finally, although graph kernels on AW have been proposed by AWE \cite{ivanov2018anonymous}, both the relationship between AWE and GNNs, and analyses on AWE remain elusive. We, therefore, consider GKs as instrumental tools in theoretically bridging AW to GNNs. 

In this paper, we propose the Graph Structural Kernel Network (GSKN), a GNN model derived from GKs that incorporates AWs and provably extends MPGNNs and the 1-WL test in terms of distinguishing graph structures. Specifically, we first design the anonymous walk graph kernel (AWGK), and derive the its GNN architecture to compute its kernel mapping. 
% following \cite{chen2020convolutional}. 
We build GSKN by combining the kernel mapping of AWGK with existing graph kernel mappings. We further analyze theoretically the ability of GSKN to distinguish graph structures from both graph-level and node-level viewpoints, on both of which GSKN is provably more powerful than MPGNNs and the 1-WL test. Correspondingly, we carry out extensive experiments on both graph and node classification tasks, GSKN outperforms various strong baselines in all scenarios, which coincides with the theoretical analysis and endorses its versatility. 

We summarize our contributions as follows: 

\begin{itemize}
    \item We propose GSKN, a GNN model based on GKs and anonymous walks that complements the inability to model substructures of MPGNNs in a theoretically principled manner. 
    \item We design the anonymous walk graph kernel (AWGK) and derive the corresponding GNN architecture to compute its kernel mapping. We further build GSKN by flexibly combining AWGK with existing graph kernel mappings.
    \item We theoretically show that GSKN possesses a stronger ability to distinguish graph structures from both graph-level and node-level viewpoints. 
    \item Graph and node classification on synthetic and real-world datasets are carried out, where GSKN outperforms various strong baselines, coinciding with the theoretical analysis. 
\end{itemize}

\section{Related Work}
\subsection{Graph Neural Networks (GNNs)}
Yet GNNs were proposed in relatively early ages \cite{scarselli2008graph}, only in recent years did we observe their tremendous success \cite{hamilton2017inductive,velivckovic2017graph}. One important reason underlying the popularity is the simplification from spectral methods \cite{chung1997spectral} to localized models \cite{defferrard2016convolutional,kipf2016semi,wu2019simplifying}, which connects GNNs to message passing \cite{gilmer2017neural,bayati2011dynamics}, or neighborhood aggregation \cite{xu2018powerful} and significantly promotes the efficiency of GNNs. Consequently, GNNs proposed henceforth largely belong to the category of message-passing GNNs (MPGNNs) \cite{chen2020can,velivckovic2017graph}. 

However, MPGNNs bear drawbacks in their ability to model substructures in graphs. For example, it is shown by \cite{morris2019weisfeiler,xu2018powerful} that the ability of MPGNNs to distinguish graphs is upper bounded by the 1-Weisfeiler-Lehman Isomorphism Test (1-WL Test). 1-WL Test is known to be blind to regular graphs, and analysis by \cite{sato2020survey} shows that numerous pairwise graphs besides regular graphs can also fool the 1-WL Test, and hence MPGNNs. As another example, \cite{chen2020can,arvind2020weisfeiler} show that MPGNNs are neither able to count subgraph isomorphisms for patterns with more than 3 nodes, nor able to detect any subgraphs other than forests of stars. These analyses point out the inherent limitations of MPGNNs on tackling graph substructures.

\subsection{Substructures in Graph Mining}
Substructures have been identified as important indicators in graph mining since \cite{granovetter1973strength}, who pointed out the significance of triads in social networks. Later works focus on more diverse sets of substructures, such as motifs \cite{milo2002network}, i.e. over-represented subgraph patterns, and graphlets \cite{prvzulj2007biological}, i.e. induced subgraph patterns. These substructures have proven to be indicative in analyzing numerous types of graphs, including biochemical networks \cite{prvzulj2007biological}, social networks \cite{juszczyszyn2008local,li2017motif,paranjape2017motifs} and even semantic segmentation \cite{zhang2013probabilistic}.

The importance of substructures motivates researchers to seek solutions that complement MPGNNs with the awareness of substructures. One line of works aims to leverage network motifs and graphlets \cite{lee19-motif-attention,bouritsas2020improving} to build substructure-aware GNNs. However, both of them rely on fixed sets of motifs and graphlets that need to be selected prior to model learning. Since real-world networks are diverse and may be characterized by varying types of substructures, the reliance on pre-selection significantly limits their applications in practice. Another line of works resort to a more flexible substructure unit called anonymous walks (AW) \cite{JinSS20,long2020graph} which do not require pre-selection, and build GNN models with impressive empirical results. However, one common drawback of them is that they fail to demonstrate their solutions compared to MPGNNs in a theoretically principled manner. 

\subsection{Graph Kernels (GKs)} 
Kernel methods have been widely studied and applied in general machine learning \cite{shawe2004kernel}. Kernel methods evaluate pairwise similarities between data samples through kernel functions, during which the data samples are implicitly projected onto higher dimensional spaces (called the \textit{Reproducing Kernel Hilbert Space}, RKHS) and thus endowed with richer features to facilitate classification. 

Kernel methods on graphs, namely Graph Kernels (GK) evaluate pairwise similarity between nodes or graphs by decomposing them into basic structural units, characteristic choices of which include random walks \cite{kang2012fast}, subtrees \cite{she2011weisfeiler}, shortest paths \cite{borgwardt2005shortest}, graphlets \cite{shervashidze2009efficient}, and also AW \cite{ivanov2018anonymous}. Such decomposition,  
% along with the connections between kernels and learning theory, 
provides principled manners to analyze the ability of GKs to express graph structures \cite{kriege2018property,chen2020convolutional,gartner2003graph}. 

Due to their appealing theoretical properties, there have also been noticeable efforts in fusing GKs with GNNs. \cite{lei2017deriving,chen2020convolutional,du2019graph} derive GNN architectures for various graph kernels, such as subtrees, random walks, paths, and their mixtures. However, regarding the AWE kernel \cite{ivanov2018anonymous}, no efforts are made to connect it with GNNs.

\section{Preliminaries}
\subsection{Notations}
We represent a graph by $G = (V, E, X)$ where $V, E, X$ denote node sets, edge sets and attribute matrices respectively. We denote the Dirac function as $\delta(\cdot, \cdot)$, where $\delta(a, b) = 1$ iff $a=b$. We use $\mathcal{W}^l(G, u)$ to denote the set of random walks with length $l$ starting from $u$ in graph $G$, and $\mathcal{P}^l(G, u)$ similarly as the set of length $l$ paths. A path $p$ does not allow duplicate nodes while a walk $w$ does. We denote $X(w)=(X_{w_1}, X_{w_2}, ...X_{w_{|w|}})$ as the concatenation of node attributes $X$ along a walk $w$, and similarly for a path $p$. 
\subsection{Graph Kernels}
Given two graphs $G_1 = (V_1, E_1, X_1)$ and $G_2 = (V_2, E_2, X_2)$, a graph kernel function $K(G_1, G_2)$ returns a similarity measure between $G_1$ and $G_2$ through the following formula: 
\begin{equation}
    K(G_1, G_2) = \sum_{u_1\in V_1}\sum_{u_2\in V_2} \kappa_{base}\left(l_{G_1}(u_1), l_{G_2}(u_2)\right)
    \label{eqn:gk}
\end{equation}
where $l_{G}(u)$ denotes a set of local substructures centered at node $u$ in graph $G$, and $\kappa_{base}$ is a base kernel computing the similarity between two sets of substructures. For simplicity we may omit $l_G(u)$ and rewrite Eqn. \ref{eqn:gk} as:
\begin{equation}
    K(G_1, G_2) = \sum_{u_1\in V_1}\sum_{u_2\in V_2} \kappa_{base}(u_1, u_2)
    \label{eqn:gk_sim}
\end{equation}
as long as the substructure set is clearly stated. We use uppercase letter $K(G_1, G_2)$ to denote graph kernels, $\kappa(u_1, u_2)$ to denote node kernels, and lowercase $k(x, y)$ to denote general kernel functions. 

The kernel mapping of a kernel $\psi$ maps a data point into the corresponding RKHS $\mathcal{H}$. Formally, given a kernel $k_{*}(\cdot, \cdot)$, then the following equation holds for its kernel mapping $\psi_{*}$,
\begin{equation}
    \forall x_1, x_2, k_*(x_1, x_2) = \langle \psi_*(x_1), \psi_*(x_2)\rangle_{\mathcal{H}_{*}},
    \label{eqn:RKHS}
\end{equation}
where $\mathcal{H}_*$ is the RKHS of $k_*(\cdot, \cdot)$.

We introduce several commonly studied graph kernels below. 

\noindent \textbf{Walk and Path Kernels.} A $l$-walk kernel $K_{walk}^{(l)}$ compares all length $l$ walks starting from each node in two graphs $G_1, G_2$,
\begin{equation}
    \begin{aligned}
        \kappa_{walk}^{(l)}(u_1, u_2) &= \sum_{w_1\in \mathcal{W}^l (G_1, u_1)}\sum_{w_2\in \mathcal{W}^l(G_2, u_2)}\delta(X_1(w_1), X_2(w_2)),\\
        K_{walk}^{(l)}(G_1, G_2) &= \sum_{u_1\in V_1}\sum_{u_2\in V_2} \kappa_{walk}^{(l)}(u_1, u_2).
    \end{aligned}
    \label{eqn:walk_gk}
\end{equation}
Substituting $\mathcal{W}$ with $\mathcal{P}$  yields a $l$-path kernel. 

\noindent \textbf{WL Subtree Kernels.} The WL subtree kernel is a finite-depth kernel variant of the 1-WL test. A WL subtree kernel of depth $l$, $K_{WL}^{(l)}$ compares all subtrees with depth $\le l$ rooted at each node. 
\begin{equation}
    \begin{aligned}
        \kappa_{subtree}^{(i)}(u_1, u_2) &= \sum_{t_1 \in \mathcal{T}^i(G_1, u_2)}\sum_{t_2\in \mathcal{T}^i(G_2, u_2)} \delta(t_1, t_2)\\
        K_{subtree}^{(i)}(G_1, G_2) &= \sum_{u_1\in V_1}\sum_{u_2\in V_2} \kappa_{subtree}^{(i)}(u_1, u_2)\\
        K_{WL}^{(l)}(G_1, G_2) &= \sum_{i=0}^{l} K_{subtree}^{(i)}(G_1, G_2),\\
    \end{aligned}
    \label{eqn:wl-kernel}
\end{equation}
where $t\in \mathcal{T}^{(i)}(G, u)$ denotes a subtree of depth $i$ rooted at $u$ in $G$. 

In reality when node attributes $X$ are of continuous values, we may replace the hard $\delta(\cdot, \cdot)$ in Eqn. \ref{eqn:walk_gk} and \ref{eqn:wl-kernel} by soft relaxations, such as the Gaussian kernel
\begin{equation}
    k_{Gauss}(x_1, x_2) = \exp\left(-\frac{\alpha}{2}\|x_1-x_2\|_2^2\right),
    \label{eqn:gaussian}
\end{equation}
whose RKHS $\mathcal{H}_{Gauss}$ is of infinite dimension.

\subsection{Anonymous Walks}
We briefly introduce anonymous walks (AWs) here and refer readers to \cite{micali2016reconstructing,ivanov2018anonymous} for further details. An anonymous walk is similar to a random walk, but with the exact identities of nodes removed. A node in an anonymous walk is represented by the first position where it appears. For example, $w_1=(0,9,8,11,9), w_2 = (3, 2, 9, 7, 2)$ are different random walks having the same anonymous walk $\phi_i=(0,1,2,3,1), i = 1, 2$. This AW indicates underlying triadic closures among 8, 9, 11 and 2, 7, 9. 

One appealing theoretical property of AWs is that one can reconstruct a complete neighborhood centered at node $v$ based on the AW distributions starting from $v$. 
\begin{theorem}[\cite{micali2016reconstructing}]
Let $B(v,r)$ be the subgraph induced by all nodes $u$ such that $dist(v,u)\leq r$ and $P_L$ be the distribution of anonymous walks of length $L$ starting from $v$, one can reconstruct $B(v,r)$ using $(P_1,...,P_L)$, where $L = 2(m+1)$, $m$ is the number of edges in $B(v,r)$.
\label{thm:aw}
\end{theorem}

We denote the set of anonymous walks of length $l$ starting from node $u$ in graph $G$ as $\Phi^l(G, u)$, and a single anonymous walk as $\phi\in \Phi^l(G, u)$, similar to the notations of paths and walks. 

\section{Model: GSKN }
In this section, we introduce our model Graph Structural Kernel Network (GSKN). Fig. \ref{fig:framework} gives an overview of GSKN. GSKN leverages two graph kernels, the random walk kernel (RWGK) which is similar to MPGNNs, and the anonymous walk kernel (AWGK) which will be derived later, carrying complementary structural information other than the RWGK. Nodes within a graph will be projected to the RKHS of both kernels, where we perform efficient multi-layer computation via the Nystr{\"o}m method and fuse both mappings. 

\begin{figure*}[htbp]
\centering
        \includegraphics[width=0.98\linewidth]{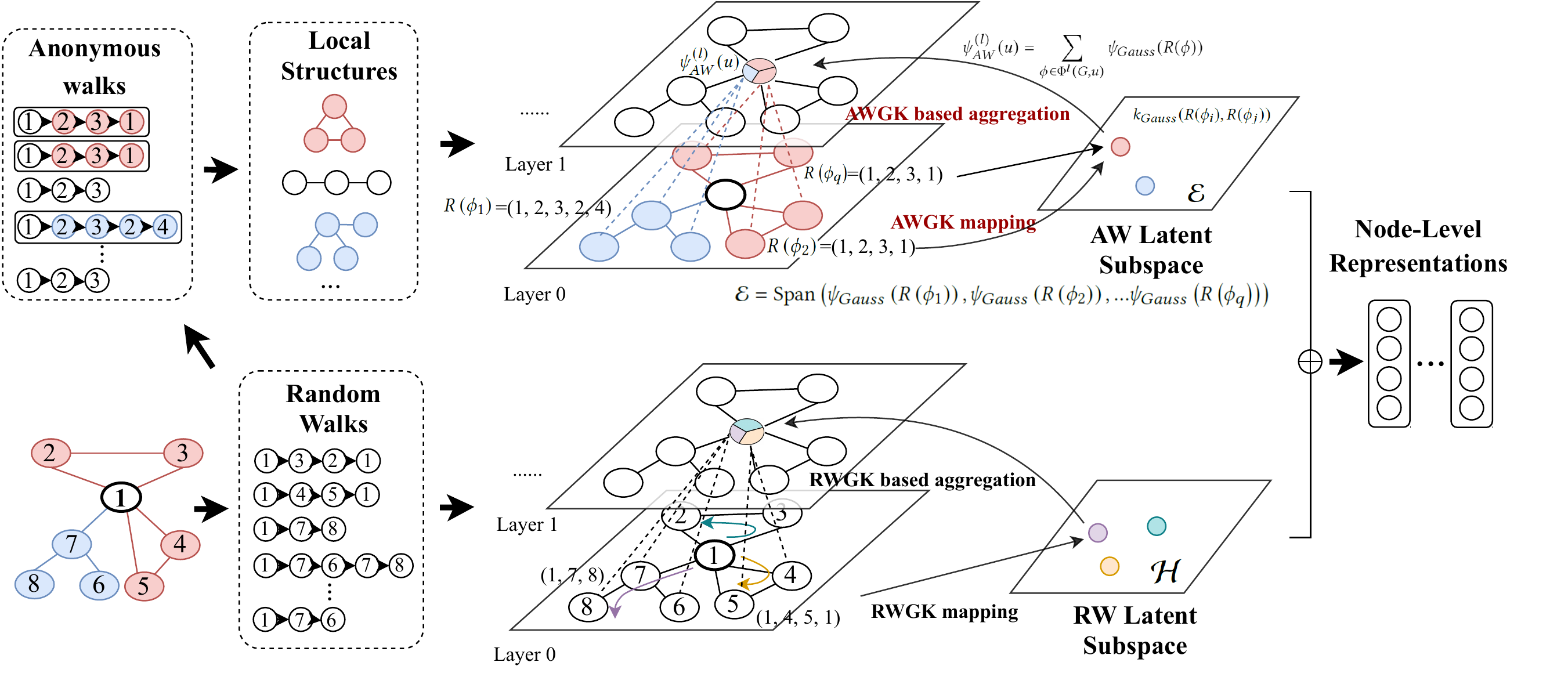}
        \caption{An overview of GSKN. GSKN combines two graph kernels, one being the RWGK, which is similar to MPGNNs, and the other being the AWGK, which captures additional structural information other than RWGK. Nodes are projected into the RKHS through the kernel mappings of both, which are formulated as GNNs and computed via the Nystr{\"o}m method. }
        \label{fig:framework}
\end{figure*}

\subsection{Anonymous Walk Graph Kernels}
We first formally extend AWs to the form of GKs. Similar to Eqn. \ref{eqn:walk_gk}, we design the $l$-anonymous walk graph kernel (AWGK) as 

\begin{equation}
    \begin{aligned}
        \kappa_{AW}^{(l)}(u_1, u_2) &= \sum_{\phi_1\in \Phi^l (G_1, u_1)}\sum_{\phi_2\in \Phi^l(G_2, u_2)}k_{Gauss}(R(\phi_1), R(\phi_2)),\\
        K_{AW}^{(l)}(G_1, G_2) &= \sum_{u_1\in V_1}\sum_{u_2\in V_2} \kappa_{AW}^{(l)}(u_1, u_2),
    \end{aligned}
    \label{eqn:anonymous_gk}
\end{equation}
where for an AW $\phi$ with length $l$, $R(\phi)\in \mathbb{R}^{l^2}$ denotes the concatenation of one-hot anonymous walk attributes along it. For example, $\phi = (1, 2, 3, 1)$, $R(\phi) = [0001, 0010, 0100, 0001]$. Since enumerating $\Phi^{l}(G, u)$ is prohibitively costly, we only sample a certain number of AWs, making $|\Phi^l(G, u)|$ constant. 
%where $\Phi^{k}(G,u)$ denotes the set of anonymous walks corresponding to random walks with length $k$ starting from $u$ in graph $G$. 
%\end{defn}

We also combine the AWGK with other kernels that can incorporate node attributes. In GSKN, we adopt random walk graph kernels (RWGK) defined in Eqn. \ref{eqn:walk_gk} due to its connection with MPGNNs \cite{chen2020convolutional,lei2017deriving}. We combine both kernels and yield the anonymous-random walk kernel (ARGK) as: 
\begin{equation}
    \begin{aligned}
        \kappa_{AR}^{(l)}(u_1, u_2) &= \kappa^{(l)}_{walk}(u_1, u_2) + \kappa_{AW}^{(l)}(u_1, u_2)\\
        K_{AR}^{(l)}(G_1, G_2) &= \sum_{u_1\in V_1}\sum_{u_2\in V_2} \kappa_{AR}^{(l)}(u_1, u_2),
    \end{aligned}
    \label{eqn:anonymous_random_gk}
\end{equation}

Combining the two kernels via addition yields a simple formulation of the kernel mapping $\psi_{AR}^{(l)}$,
\begin{equation}
    \psi_{AR}^{(l)}(u) = \left[\psi_{walk}^{(l)}(u)\big\|\psi_{AW}^{(l)}(u)\right],
   \label{eq:combine_kernels}
\end{equation}
where $[x\|y]$ denotes vector concatenation. The same holds when we replace $u$ with $G$. 

% \subsection{Fast Approximate for AWGK}
% The computation of the proposed kernel requires comparing all pairs of paths in each pair of graphs. \cite{williams2001using} proposed an approximate kernel embeddings, the Nystr{\"o}m method.
% For implementation, learning the “filters” with Nystr{\"o}m can be achieved by simply running a K-means algorithm on path attributes extracted from training
% data \cite{zhang2008improved}.
% After applying the Nystr{\"o}m method, the approximate feature map in Eqn. \ref{eq:node_embedding} becomes
% \begin{equation}
%     \phi_1(u) = \sigma_1(Z^TZ)^{-\frac{1}{2}}c_k(u)
% \end{equation}

\subsection{From AWGK to GNNs}
Kernel methods \cite{zhang2018retgk,williams2001using} implicitly perform projections from original data spaces to their RKHS $\mathcal{H}$, as shown in Eqn. \ref{eqn:RKHS}. Hence, as GNNs also project nodes or graphs into vector spaces, connections have been established between GKs and GNNs through kernel mappings. In this section, we follow \cite{chen2020convolutional} and analyze the kernel mapping of the AWGK $\psi_{AW}$ to derive the corresponding GNN formulation, which serves as an important building block of GSKN. 

Specifically, denoting $\psi_{Gauss}$ as the kernel mapping of the Gaussian kernel (Eqn. \ref{eqn:gaussian}), the kernel mapping of $\kappa_{AW}^{(l)}$ can be written as 
\begin{equation}
    \psi_{AW}^{(l)}(u) = \sum_{\phi\in \Phi^l(G, u)} \psi_{Gauss}(R(\phi)).
    \label{eqn:psi-aw}
\end{equation}

However, as $\mathcal{H}_{Gauss}$ is of infinite dimension, it is impossible to evaluate Eqn. \ref{eqn:psi-aw} accurately. Alternatively, we resort to the Nystr{\"o}m Method \cite{williams2001using} for a finite-dimensional approximation. 

Nystr{\"o}m method aims to project points from an arbitrary RKHS into finite-dimensional subspaces. Specifically, given a collection of $q$ anonymous walk `landmarks' $\Phi_q=\{R(\phi_1), R(\phi_2),..., R(\phi_{q})\}$, 
the corresponding $q$-dimensional subspace is taken as 
\begin{equation}
    \mathcal{E} = \mathrm{Span}\left(\psi_{Gauss}\left(R\left(\phi_1\right)\right), \psi_{Gauss}\left(R\left(\phi_2\right)\right), ...\psi_{Gauss}\left(R\left(\phi_q\right)\right)\right),
\end{equation}
and the projection of another AW $R(\phi')$ onto $\mathcal{E}$ can be done by \cite{chen2019biological}: 
\begin{equation}
    \begin{aligned}
    \psi_{Gauss}(R(\phi')) =&[k_{Gauss}(R(\phi_i),R(\phi_j))]_{ij}^{-\frac{1}{2}} \\
    & [k_{Gauss}(R(\phi_1),R(\phi')),...k_{Gauss}(R(\phi_{q}),R(\phi'))]^T,
    \end{aligned}
    \label{eqn:psi_aw_nystrom1}
\end{equation}
where $[k_{Gauss}(R(\phi_i),R(\phi_j))]_{ij}$ is a $q\times q$ positive-semidefinite matrix formed by the kernel values between the $q$ landmarks. 

Note that 
\begin{equation}
    \begin{aligned}
    k_{Gauss}(R(\phi_i), R(\phi_j)) &=\exp\left(-\frac{\alpha}{2}\|R(\phi_i)-R(\phi_j)\|_2^2\right)\\
    &= \exp\left(\alpha R(\phi_i)^TR(\phi_j)-\alpha l\right)\\
    &= \sigma\left(R(\phi_i)^TR(\phi_j)\right),
    \end{aligned}
    \label{eqn:gaussian-sigma}
\end{equation}
which is essentially a dot product followed by a non-linear transformation $\sigma(x) = \exp(\alpha x - \alpha l)$. Therefore, denoting $Z = \left[R(\phi_i) \right]_{i}$, Eqn. \ref{eqn:psi_aw_nystrom1} can be rewritten as 
\begin{equation}
    \psi_{Gauss}(R(\phi')) = \sigma\left(Z^TZ\right)^{-1/2}\sigma\left(Z^TR(\phi')\right) \in \mathbb{R}^q,
\end{equation}
and consequently, Eqn. \ref{eqn:psi-aw}, the kernel mapping of the AWGK on nodes, can be approximately computed as
\begin{equation}
\psi_{AW}^{(l)}(u) = \sum_{\phi\in \Phi^l(G, u)}\sigma\left(Z^TZ\right)^{-1/2}\sigma\left(Z^TR(\phi)\right) \in \mathbb{R}^q,
\label{eqn:aw-gnn-node}
\end{equation}
which can be essentially interpreted as a one-layer GNN, consisting of linear transformation of anonymous walks $Z^TR(\phi)$, non-linearity $\sigma$, and sum-pooling\footnote{In practice, we add $ \epsilon = \text{1e-7}$ to the diagonal, i.e. $(\sigma(Z^TZ) + \epsilon I)^{-1/2}$ for better numerical stability.}. Following \cite{chen2020convolutional} we can also stack multiple layers of Eqn. \ref{eqn:aw-gnn-node} to resemble a $L$-layer GNN,
\begin{equation}
    \psi_{AW, L}^{(l)} = \underbrace{\psi_{AW, 1}^{(l)}\circ \psi_{AW, 2}^{(l)}\circ ...}_{L \text{ layers}}.
\end{equation}

The kernel mapping for $K_{AW}^{(l)}(G_1, G_2)$ can be thus computed through a sum pooling over all nodes, 
\begin{equation}
    \Psi_{AW, L}^{(l)}(G) = \sum_{u\in V} \psi_{AW, L}^{(l)} (u),
    % \forall u\in V.
    \label{eqn:aw-gnn-graph}
\end{equation}
which is also similar to GNNs. 

Finally, we combine the kernel mappings of AWGK and RWGK via concatenation (Eqn. \ref{eq:combine_kernels}) to get node-level embeddings
\begin{equation}
    \psi_{AR, L}^{(l)}(u) = \left[\psi_{walk, L}^{(l)}(u)\big\| \psi_{AW, L}^{(l)}(u)\right],
    \label{eq:gcn}
\end{equation}
where $\psi_{walk, L}^{(l)}(u)$ is the $L$ layer GCKN-walk defined in \cite{chen2020convolutional}, which basically replaces $AW$ with $walk$ in Eqn. \ref{eqn:aw-gnn-node}. Replacing $u$ with $G$, $\psi$ with $\Psi$ yields graph-level embeddings. \footnote{Note that $\psi_{walk}^{(l)}$ can be replaced by any other kernel mapping which leverages node attributes. We take the RWGK for simplicity.  }

\subsection{Learning the Model}
The model parameters are essentially $Z_{walk}$ and $Z_{AW}$, i.e. the landmarks for the finite-dimensional approximation, which are equivalent to weights in common GNNs. Since we adopt kernel methods, we adopt an unsupervised method for selecting the $Z$. 

\cite{zhang2008improved} shows that the landmarks for Nystr{\"o}m Methods can be efficiently chosen by running $k$-means clustering on the data, and \cite{chen2020convolutional} shows that such techniques are also empirically good under graph kernels. Therefore, we learn the $Z_{AW}$ by running $q$-means on the anonymous walks $R(\phi)$, and similarly for $Z_{walk}$. 

% \noindent \textbf{Supervised Learning} The $Z$ can also be learned by backpropagating supervised learning loss functions, both node-level and graph-level. 

\noindent \textbf{Complexity}
The complexity of computing the embeddings of GSKN consists of three parts: walk enumeration, feature transformation, and $q$-means training. Finding all length $l$ walks can be done using Depth-First Search (DFS), whose worst-case complexity for each graph is $O(|V|d^l)$, where $d$ is the maximum degree. The complexity for sampling anonymous walks is $O(|V|\cdot |\Phi^l(G,u)|l)$. 

For feature transformation, denoting the input and output dimensions as $q_0$, $q_1$ respectively, each walk is encoded in $q_1$ inner products, each with a complexity of $l\cdot q_0$, yielding $O(l\cdot q_1q_0)$. Similarly, each anonymous walk is encoded in $O(q_1\cdot l^2)$ operations. 

For $q$-means training, the complexity for heuristically solving a $q$-means on $n\times d$ data would be $O(qnd)$. Therefore, for walks and AWs the complexity would be $O(|V|d^lq_0q_1\cdot l)$, $O(|V|\cdot|\Phi^l(G, u)|l^2q_1))$ respectively. Therefore, the total complexity would be asymptotically dominated by $O(|V|d^l)$, which is irrelevant to AW. Thus, we conclude that the efficiency of GSKN should not be compromised by incorporating AW. 

Considering the fact that the complexity is exponential to $l$ for RW, but only polynomial for AW, in practice we use different $l$ for RW and AW, which will be detailed in the experiments. 

We show the pseudo-code of GSKN in Algorithm \ref{alg:GSKN}. 

\begin{algorithm}[tb]
\caption{Algorithm of GSKN}
\label{alg:GSKN}
\begin{algorithmic}[1] %[1] enables line numbers
\REQUIRE Graph $G=(V,E,X)$
\ENSURE Node-level embeddings $\psi_{AR}^{(l)}$, graph-level embeddings $\Psi_{AR}^{(l)}$

\FORALL {nodes $u \in V$}
    \STATE Generate random walk sequences $\mathcal{W}^l(G, u)$.
    \STATE Generate anonymous walk sequences $\Phi^l(G, u)$.
\ENDFOR
\FOR {layer $k=0, ..., L$}
    \FORALL {center node $u \in V$}
        \STATE Compute kernel mapping of AWGK  (Eqn. \ref{eqn:aw-gnn-node}).
        \STATE Compute kernel mapping of RWGK. 
        \STATE Node-Level Embedding $\psi_{AR, k}^{(l)}(u)$ (Eqn. \ref{eq:gcn}) 
    \ENDFOR
    \STATE Unsupervised Learning of $Z$ with Nystr{\"o}m method. 
\ENDFOR

\STATE Graph-Level Embedding $\Phi(G)=\sum_{u\in V}\psi_{AR, L}^{(l)}(u)$ 
\RETURN $\psi_{AR, L}^{(l)}(u), u\in V$, $\psi_{AR, L}^{(l)}(G)$. 
\end{algorithmic}
\end{algorithm}

\section{Theoretical Analysis}
\label{sec:theory}
In this section, we provide theoretical analyses regarding the ability of GSKN to distinguish different graph structures. Corresponding to the previous formulations, we analyze from two viewpoints, both graph-level, and node-level. 

We begin our analysis from the graph-level analysis and reach the following conclusion. 
\begin{theorem}
Suppose the feature space $\mathcal{X}$ is finite, and define the space of graphs with feature space $\mathcal{X}$ as $\mathcal{G} = \{G=(V, E, X), X\subset \mathcal{X}\}$. Given an arbitrary graph $G\in \mathcal{G}$, there exists parameterization of GSKN such that it is as powerful as the 1-WL test and hence MPGNNs. Moreover, There exists graphs $G'\in \mathcal{G}$ such that GSKN can learn strictly more powerful functions than the 1-WL test. 
\label{thm:graph}
\end{theorem}

We would like to point out that although the finite feature space assumption is weaker than that of \cite{xu2018powerful}, it is still adequate in practice. For categorical attributes (e.g. one-hot, discrete, bag-of-words) with finite dimensions, the feature space is always finite. For continuous node attributes, in practice, we always discretize them into bins (e.g. through floating-point numbers), yielding a finite feature space.

We prove the theorem in the following procedure. 
\begin{enumerate}
    \item First, we show that the layer-wise update rule for both $\psi_{AW}^{(l)}$ and $\psi_{walk}^{(l)}$ are injective w.r.t input attributes in Lemma \ref{lem:injective_layer}. Upon it, we show the power of GSKN in Corollary \ref{cor:lower_bound}. 
    \item In addition, we show that GSKN can distinguish graph structures while 1-WL and MPGNNs cannot in Lemma \ref{lem:strict}. 
\end{enumerate}

\begin{lem}
Given $f(x) = \left[f_1(x)\|f_2(x)\right]$ and $x_1\neq x_2$, $f(x_1)\neq f(x_2)$ if $f_1(x_1)\neq f_1(x_2)$ or $f_2(x_1)\neq f_2(x_2)$. 
\label{lem:concat}
\end{lem}
\begin{proof}
The result comes from $x\neq y\Rightarrow [x\|z]\neq [y\|z]$ and is hence trivial. 
\end{proof}

Since $\Psi_{AR, L}^{(l)}(G) = \left[\Psi_{walk, L}^{(l)}(G)\LARGE\|\Psi_{AW, L}^{(l)}(G)\right]$, $G_1$ and $G_2$ will be concluded different if either $\Psi_{walk, L}^{(l)}$ or $\Psi_{AW, L}^{(l)}$ deems them different. Therefore, the ability of $\Psi_{AR, L}^{(l)}$ will be lower bounded by both $\Psi_{walk, L}^{(l)}$ and $\Psi_{AW, L}^{(l)}$. We then show the ability of both of them by analyzing its layer-wise update function. 

\begin{lem}
Consider function $g(x) = Wx$, $x\in X$, and $X\subset \mathbb{R}^a$ is finite. $W\in \mathbb{R}^{b\times a}$. Then $\exists W$, such that $g(x)$ is injective w.r.t $x$.
\label{lem:injective_linear}
\end{lem}
\begin{proof}
$g(x)$ is injective $\Leftrightarrow$ $\forall x\neq x'\in X, g(x) \neq g(x')\Leftrightarrow \forall x\neq x'\in X, x-x' \notin \mathrm{Ker}W$. When $a<b$, then as long as $\mathrm{rank}(W) = a$, $\mathrm{Ker}W = \{0\}$, and therefore $\forall x\neq x', x-x'\notin \mathrm{Ker}W$. 

When $a<b$, notice that $\overline{X} = \{x-x'|x\neq x'\in X\}$ is also finite. Since the set $\{\mathrm{Ker}W|W\in\mathbb{R}^{b\times a}\}$ is infinite, then there always exists $W$, such that $\mathrm{Ker} W\cap \overline{X} = \emptyset$. 
\end{proof}

\begin{lem}
Consider the layer-wise update function in Eqn. \ref{eqn:aw-gnn-node}. There exists parameter $Z$ such that the function is injective w.r.t $\Phi^{l}(G, u)$ (and similarly $\mathcal{W}^l(G, u)$). 
\label{lem:injective_layer}
\end{lem}
\begin{proof}
We first prove that there exist $Z$ such that 
\begin{equation}
\sigma\left(Z^TZ\right)^{-1/2}\sigma\left(Z^Tx\right)
\label{eqn:inner}
\end{equation}
is injective to $x$, and Lemma \ref{lem:injective_layer} follows by Lemma 5 in \cite{xu2018powerful}.

According to Lemma \ref{lem:injective_linear}, there exists $Z$ such that $Z^Tx$ is injective w.r.t $x\in X$ from finite feature space. Also, $\sigma(\cdot)$ is element-wise monotonic and therefore injective, and $\sigma\left(Z^TZ\right)^{-1/2}$ is of full rank (because we add a small $\epsilon$ to its diagonal in practice), thereby multiplying by it is injective. 

We thus conclude that $\exists Z$ such that Eqn. \ref{eqn:inner} is injective, and the injectiveness of Eqn. \ref{eqn:aw-gnn-node} is ensured following Lemma 5 in \cite{xu2018powerful}. 
\end{proof}
\begin{corollary}
GSKN is at least as powerful as the 1-WL test, and hence the strongest MPGNNs. 
\label{cor:lower_bound}
\end{corollary}
\begin{proof}
According to Theorem 3 in \cite{xu2018powerful}, a GNN is as strong as the 1-WL test as long as the layer-wise neighborhood aggregation for node features is injective, and the graph readout function is injective. According to Lemma \ref{lem:injective_layer} and the fact that neighbors can be viewed as length 2 walks, the layer-wise injective nature can be satisfied. Also, similar to the arguments on Eqn. 4.2 in \cite{xu2018powerful}, the graph readout is also injective. 
\end{proof}
We then show that $\Psi_{AR, L}^{(l)} = \left[\Psi_{walk, L}^{(l)}\|\Psi_{AW, L}^{(l)}\right]$ is able to learn functions beyond the 1-WL test. 
\begin{lem}
There exists graphs $G$ such that $\Psi_{AR, L}^{(l)}$ is strictly more powerful than the 1-WL test. 
\label{lem:strict}
\end{lem}
\begin{proof}
For proving Lemma \ref{lem:strict} it suffices to provide such examples $G_1, G_2$ where 1-WL test deems them isomorphic and $\Psi_{AR, L}^{(l)}(G_1)\neq \Psi_{AR, L}^{(l)}(G_2)$. Rings would be one type of such examples, as shown in Fig. \ref{fig:1-wl-example}. \cite{damke2020novel} shows that 1-WL cannot distinguish between $R_{2k}$, a ring with $2k$ nodes and $R_{k, k}$, two disjoint rings with $k$ nodes each. 

We show that $\Psi_{AW, L}^{(l)}$ can distinguish such pairs of graphs, and $\Psi_{AR, L}^{(l)}$ follows. Since any walk on $R_{k, k}$ would pass at most $k$ distinct nodes, an arbitrary anonymous walk would have a label of at most $k$. Therefore, it is sufficient to take $l > k$, where on $R_{2k}$ anonymous walks with label larger than $k$ would be observed. 
\end{proof}

\begin{figure}[htbp]
\centering
        \includegraphics[width=\linewidth]{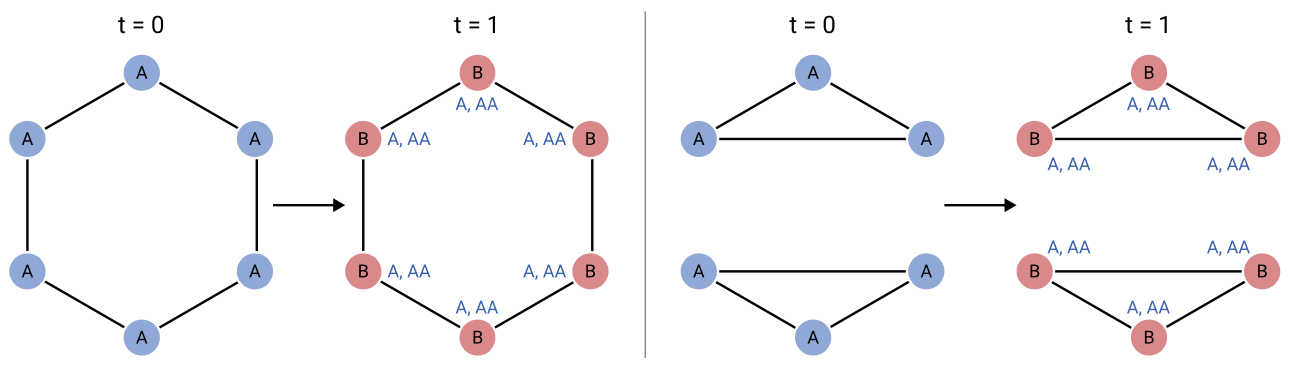}
        \caption{Two simple non-isomorphic graphs that are indistinguishable by 1-WL \cite{damke2020novel}. $t$ is the number of iterations.}
        \label{fig:1-wl-example}
\end{figure}

In addition, we extend our analysis to node-level viewpoints, and show that our model provably extends the WL subtree kernel (Eqn. \ref{eqn:wl-kernel}) in differentiating different rooted local subgraphs.  
\begin{theorem}
Given a graph $G = (V, E, X)$ and $u_1,  u_2 \in V$. Denote $\mathcal{M}(u_1, u_2)$ as the set of exact matchings between subsets of neighborhoods of $u_1$ and $u_2$ (formally defined in \cite{she2011weisfeiler}). For $u_1, u_2$ such that $|\mathcal{M}(u_1, u_2)| = 1$, the following inequality holds. 
\begin{equation}
\begin{aligned}
    \kappa_{subtree}^{(l)}(u_1, u_2) &=\delta\left(\psi_{walk}^{(l)}(u_1), \psi_{walk}^{(l)}(u_2)\right)\\
    &\ge \delta\left(\psi_{AR}^{(l)}(u_1), \psi_{AR}^{(l)}(u_2)\right)
    \end{aligned}
    \label{eqn:node}
\end{equation}
Moreover, there exists $G, u_1, u_2$ such that strict inequality holds. 
\label{thm:node}
\end{theorem}

\noindent \textbf{Remarks.} Both the WL subtree kernel and the Delta function $\delta(\cdot, \cdot)$ return only 0 or 1, and therefore they are indicators of whether they deem local subgraphs rooted at $u_1$ and $u_2$ isomorphic. $\delta(\psi_1(u_1), \psi_1(u_2)) > \delta(\psi_2(u_1), \psi_2(u_2))$ indicates $\delta(\psi_1(u_1), \psi_1(u_2)) = 1$, while $\delta(\psi_2(u_1),$ $\psi_2(u_2)) = 0$, which indicates that $\psi_2$ is more powerful than $\psi_1$. 

\begin{proof}
The first equality is shown by Theorem 1 in \cite{chen2020convolutional}. The second inequality is a direct corollary of Lemma \ref{lem:concat}. 

To show that there exists $G, u_1, u_2$ such that strict inequality holds, we revisit the example in Fig. \ref{fig:1-wl-example}. Suppose $G$ consists of disjoint rings (i.e. 2-regular graphs), $u_1$ is in an $R_{2k}$, while $u_2$ is in an ($R_k$), and node attributes are node degrees $X(u) = \mathrm{deg}(u)$. Similar to Lemma \ref{lem:strict}, AWs with length $l>k$ would be sufficient such that $\delta\left(\psi_{AW}^{(l)}(u_1), \psi_{AW}^{(l)}(u_2)\right) = 0$. However, since node attributes and node degrees are all the same, the kernel mapping of the RWGK $\psi_{walk}^{(l)}$ will follow $\delta\left(\psi_{walk}^{(l)}(u_1),  \psi_{walk}^{(l)}(u_2)\right) = 1, \forall l$.  
\end{proof}

\begin{table*}[h]
  \centering
  \resizebox{\linewidth}{!}{\begin{tabular}{c c c c c c c c c c c c c}
  \toprule
  Datasets &
  {MUTAG} &
  {PROTEINS} &
  {PTC} &
  {IMDB-B} &
  {IMDB-M} & 
  {COLLAB} & 
  {BZR} & 
  {COX2} & 
  {PROTEINS\_full} &
  {Cora} &
  {Pubmed} &
  {PPI}
  \\
  \hline
  Task & Graph & Graph & Graph & Graph & Graph & Graph & Graph & Graph & Graph & Node & Node & Node\\
  Type & biochem & biochem & biochem & social & social & social & biochem & biochem & biochem & citation & citation & biochem\\
  \# Graphs & 188 & 1,113 & 344 & 1,000 & 1,500 & 5,000 & 405 & 467 & 1,113 & - & - & -  \\
  Avg \# node  & 18 & 39 & 26 & 20 & 13 & 74 & 36 & 41 & 39 & 2,708 & 19,717 & 14,755  \\
 Avg \# edge   & 20 & 73 & 51 & 97 & 66 & 2,458 & 38 & 43 & 73 & 5,429 & 44,338 & 228,431\\
  \# Class  & 2 & 2 & 2 & 2 & 3 & 3 & 2 & 2 & 2 & 7 & 3 & 121\\
  Attribute Type & Disc. & Disc. & Disc. & No & No & No & Cont. & Cont. & Cont. & Disc. & Disc. & Disc.\\
  Attr. Dim. & 1 & 1 & 1 & - & - & - & 3 & 3 & 29 & 1,433 & 500 & 50\\
  \bottomrule
    \end{tabular}}
    \caption{Dataset statistics for graph and node classification. \textit{Disc.} stands for discrete node attributes, while \textit{Cont.} stands for continuous node attributes.}
  \label{tab:datasets1}
\end{table*}

\section{Experiments}
In this section, we introduce our empirical evaluations on GSKN. We first introduce experimental settings, before presenting various experimental results. Specifically, our evaluations consist of: 
\begin{itemize}
    \item Qualitative evaluation, where we carry out experiments on synthetic datasets to illustrate properties of GSKN in an intuitive manner. 
    \item Quantitative evaluation, including graph classification and node classification on real-world graph datasets. 
    \item Self evaluation, including analysis on model components, model parameters, and model efficiency. 
\end{itemize}
\subsection{Experimental Setup}
We first introduce the datasets, comparison methods as well as settings for experimental evaluation. 

\noindent \textbf{Datasets} We take the following benchmark datasets for evaluation. 
\begin{itemize}
    \item \textbf{Graph Classification}. We use the same
benchmark datasets as in \cite{du2019graph}, including 6 biochemical network datasets (MUTAG, PROTEINS, PTC, BZR, COX2, PROTEINS\_full), and 3 social network datasets (IMDB-B, IMDB-M, and COLLAB). All biochemical network datasets have node attributes, with MUTAG, PROTEINS, PTC categorical and BZR, COX2, PROTEINS\_full continuous, while none of the social network datasets have node features. 
   \item \textbf{Node Classification}. We utilize citation networks (Cora, Pubmed \cite{kipf2016semi}), and the protein interaction network (PPI). 
\end{itemize}
Detailed dataset statistics and properties are listed in Table \ref{tab:datasets1}.

\noindent \textbf{Baselines} 
We take the following models as baselines for graph classification. 
\begin{itemize}
    \item \textbf{LDP} \cite{cai2018simple} serves as a simple baseline based on degree information only. 
    \item \textbf{Graph Kernels}, including the
    WL subtree kernel \cite{she2011weisfeiler}, the Shortest Path (SP) Kernel \cite{borgwardt2005shortest}, the RW kernel,  GNTK \cite{du2019graph}, RetGK \cite{zhang2018retgk}, DDGK\cite{al2019ddgk}, hashing graph kernels (HGK-WL, HGK-SP) \cite{morris2016faster}, Graph2Vec \cite{narayanan2017graph2vec}, and WWL \cite {togninalli2019}.
    \item \textbf{GNNs}, including PatchySAN \cite{niepert2016learning}, GraphSAGE \cite{hamilton2017inductive}, GCN \cite{kipf2016semi}, GIN \cite{xu2018powerful}. 
    \item \textbf{AW based methods. } We choose AWE \cite{ivanov2018anonymous} and GraphSTONE \cite{long2020graph}. AWE is based on skip-gram optimization on AW while GraphSTONE is an AW based GNN. 
    \item \textbf{Graph Kernel Networks, } i.e. GCKN. We choose GCKN-Path, the strongest variant mentioned in \cite{chen2020convolutional}. 
\end{itemize}
For node classification, we take GNNs as baselines, including GCN, GraphSAGE, GAT, GIN, GraLSP \cite{JinSS20} and GraphSTONE. Among them, GraLSP and GraphSTONE also adopt AW to model substructure information. 

\noindent \textbf{Hyperparameter Settings} We take 64-dimensional embeddings for all methods. The parameter settings of other baselines follow the recommended settings in relevant codes. 
% For GNNs, we take 2-layer networks with a hidden layer sized 100. For skip-gram optimization, we take $N = 100, l = 10$, window size as 5 and the number of negative sampling as 8. For models involving neighborhood sampling, we take the number for sampling as 20. 
% In addition, 
For GSKN, we take $\alpha = 1.5$ in Eqn. \ref{eqn:gaussian} and \ref{eqn:gaussian-sigma}, and choose anonymous walk length to be 5 or 6 depending on each dataset. Since the embeddings of GSKN come from the concatenation of kernel mappings of both AWGK and RWGK, we set the embedding size of each part as 32. 
% Codes of our model GSKN are available at \href{www.google.drive.com}{link}.

% \begin{table}[htbp]
%   \centering
%   \resizebox{\columnwidth}{!}{
%   \begin{tabular}{ccccccc}
%   \toprule
%   Datasets &
%   {Type} &
%   {\# Graphs} &
%   {Avg \# node} &
%   {Avg \# edge} &
%   {\# Class}\\
%   \hline
%   MUTAG & biochem & 188 & 18 & 20 & 2 \\
%   PROTEINS & biochem & 1,113 & 39 & 73 & 2 \\
%   PTC & biochem & 344 & 26 & 51 & 2 \\
%   IMDB-B & social & 1,000 & 20 & 97 & 2 \\
%   IMDB-M & social & 1,500 & 13 & 66 & 3\\
%   COLLAB & social & 5,000 & 74 & 2458 & 3\\
%   \midrule
%   BZR & & 405 & 35.8 & 38.3 \\
%   COX2 \\
%   PROTEINS\_full \\
%   \bottomrule
%     \end{tabular}}
%     \caption{Dataset statistics for graph classification.}
%   \label{tab:datasets1}
% \end{table}

%\begin{table}[htbp]
% \centering
% \begin{tabular}{cccc}
% \toprule
% Type &
% Number &
% {$|V|$} &
% {$|E|$} \\
% \hline
% \textit{Reg-20-5} & 50 & 20 & 50 \\
% \textit{2-Reg-10-5} & 50 & 20 & 50   \\
% \bottomrule
%    \end{tabular}
%    \caption{Statistics of synthetic graph instances in \textit{Regular}.}
% \label{tab:datasets4}
%\end{table}

\subsection{Synthetic Experiments for Proof-of-concept}
\begin{figure}[ht]
    \centering
    \subfigure[Cycle]{\includegraphics[width=0.23\columnwidth]{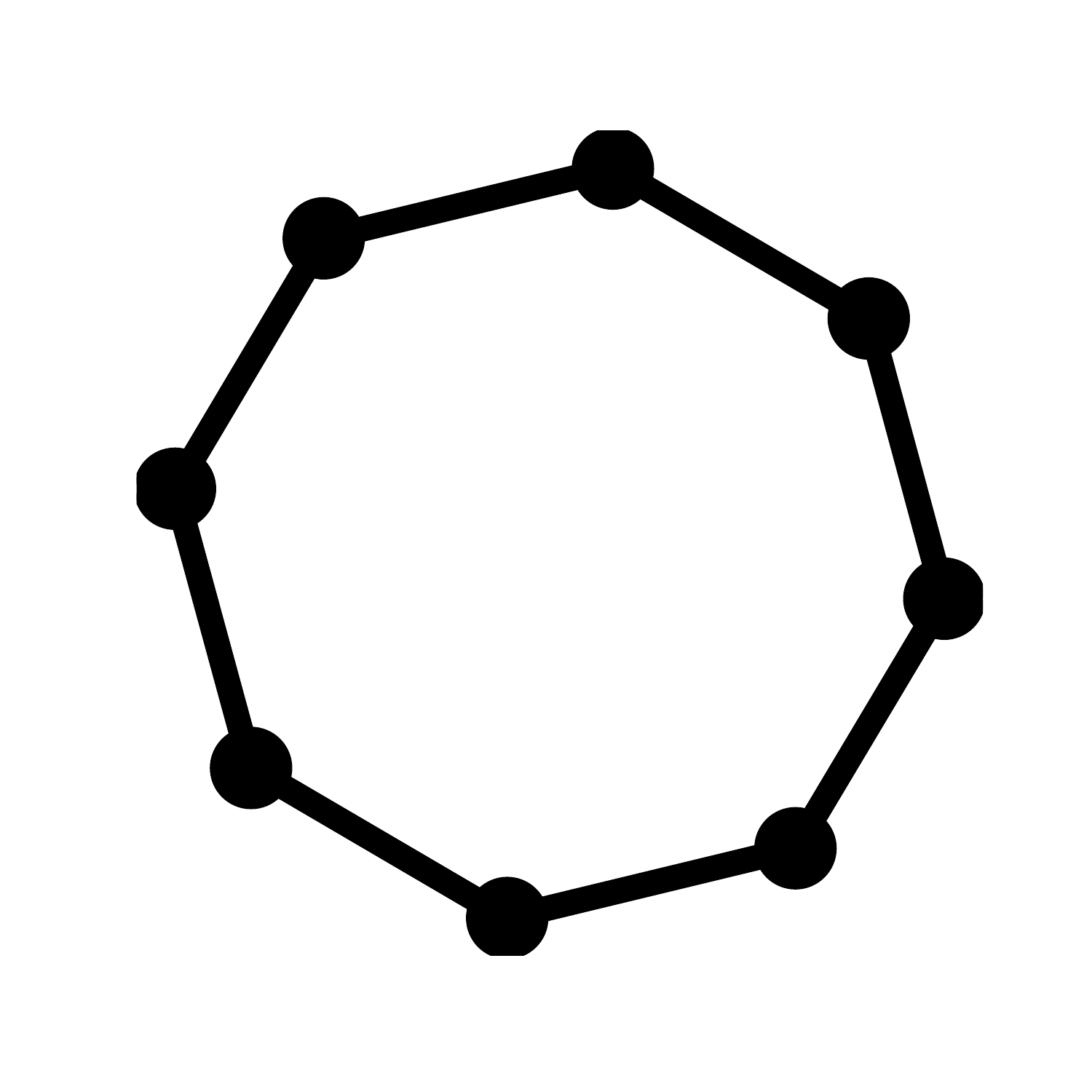}}
    \subfigure[Wheel]{\includegraphics[width=0.23\columnwidth]{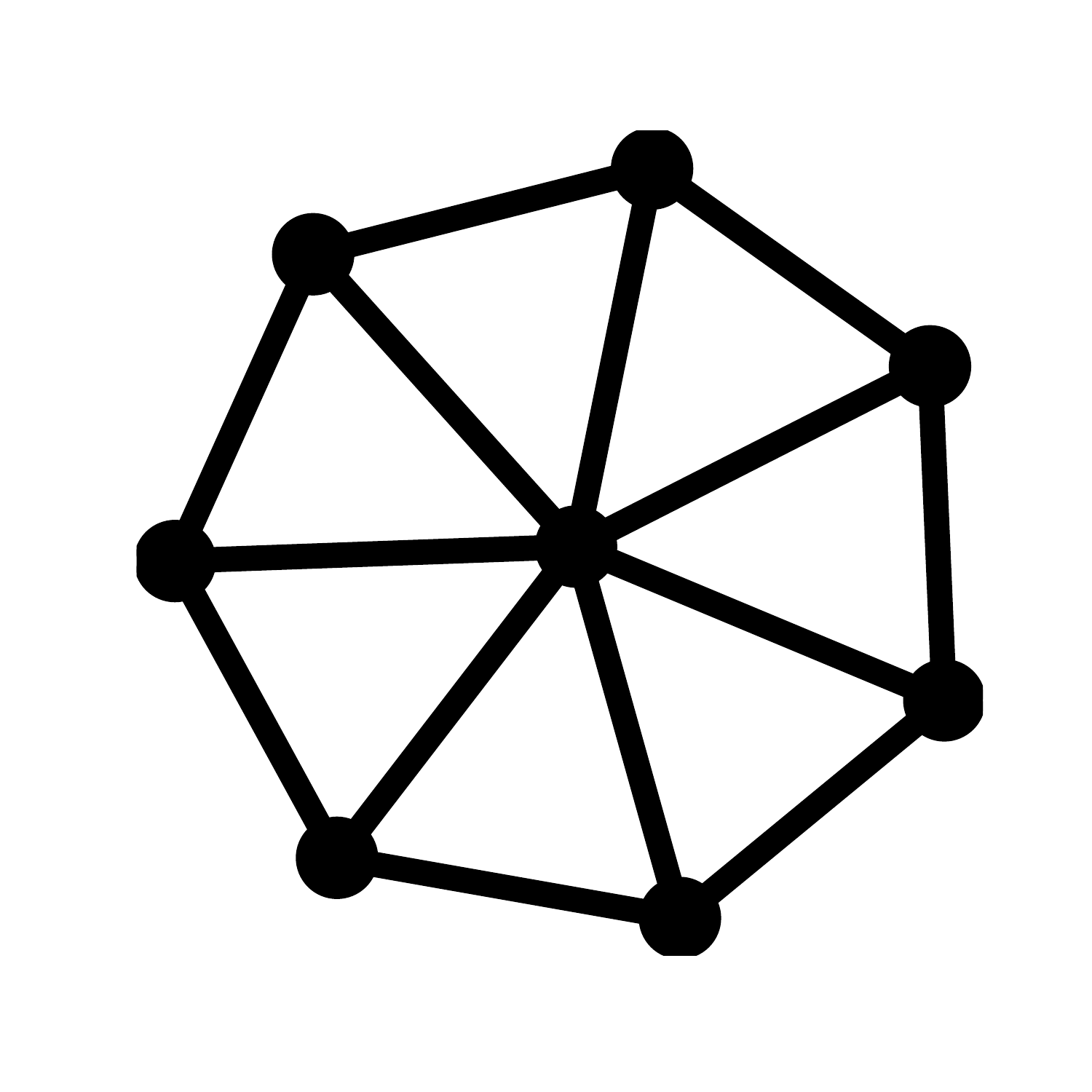}}
    \subfigure[Path]{\includegraphics[width=0.23\columnwidth]{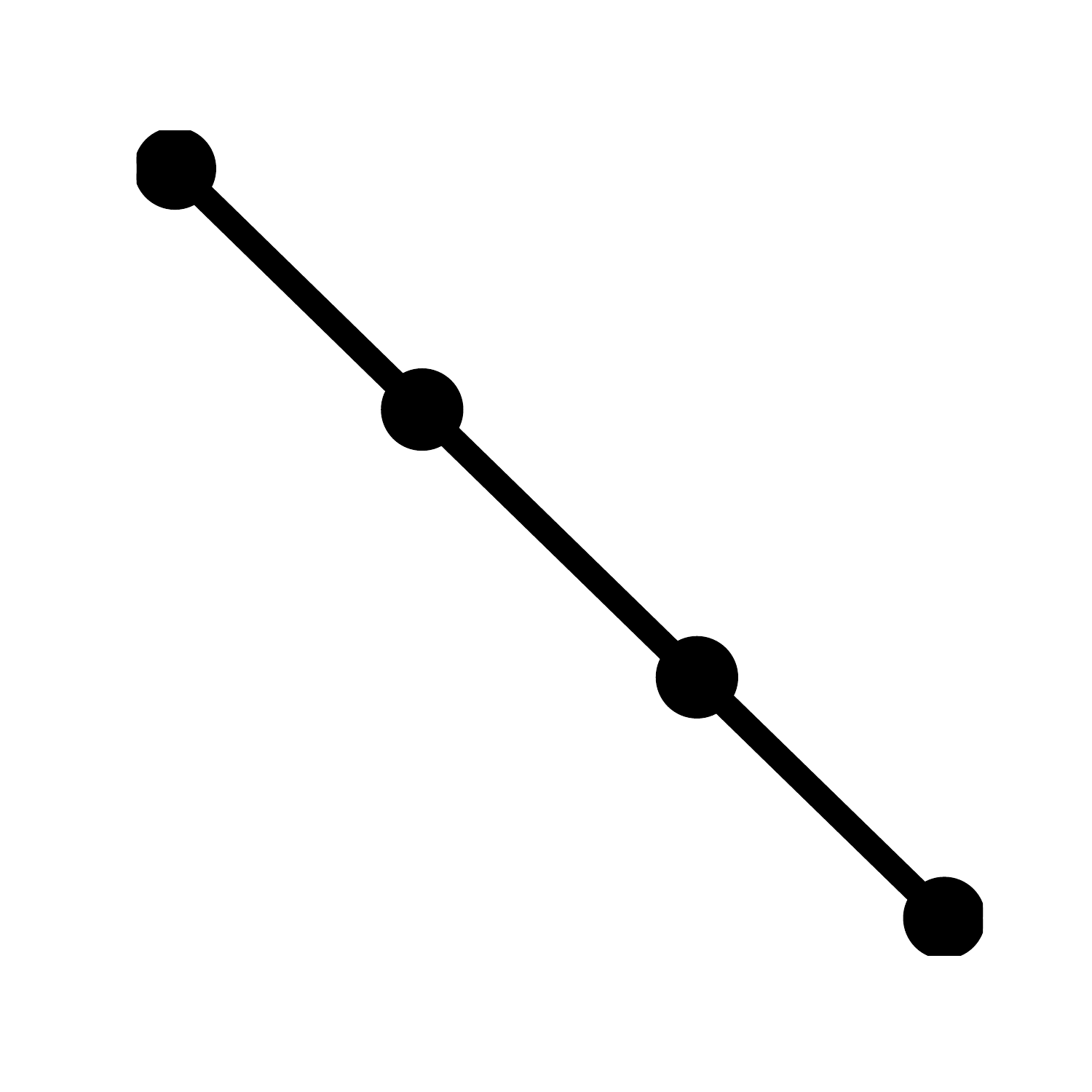}}
    \subfigure[Ladder]{\includegraphics[width=0.23\columnwidth]{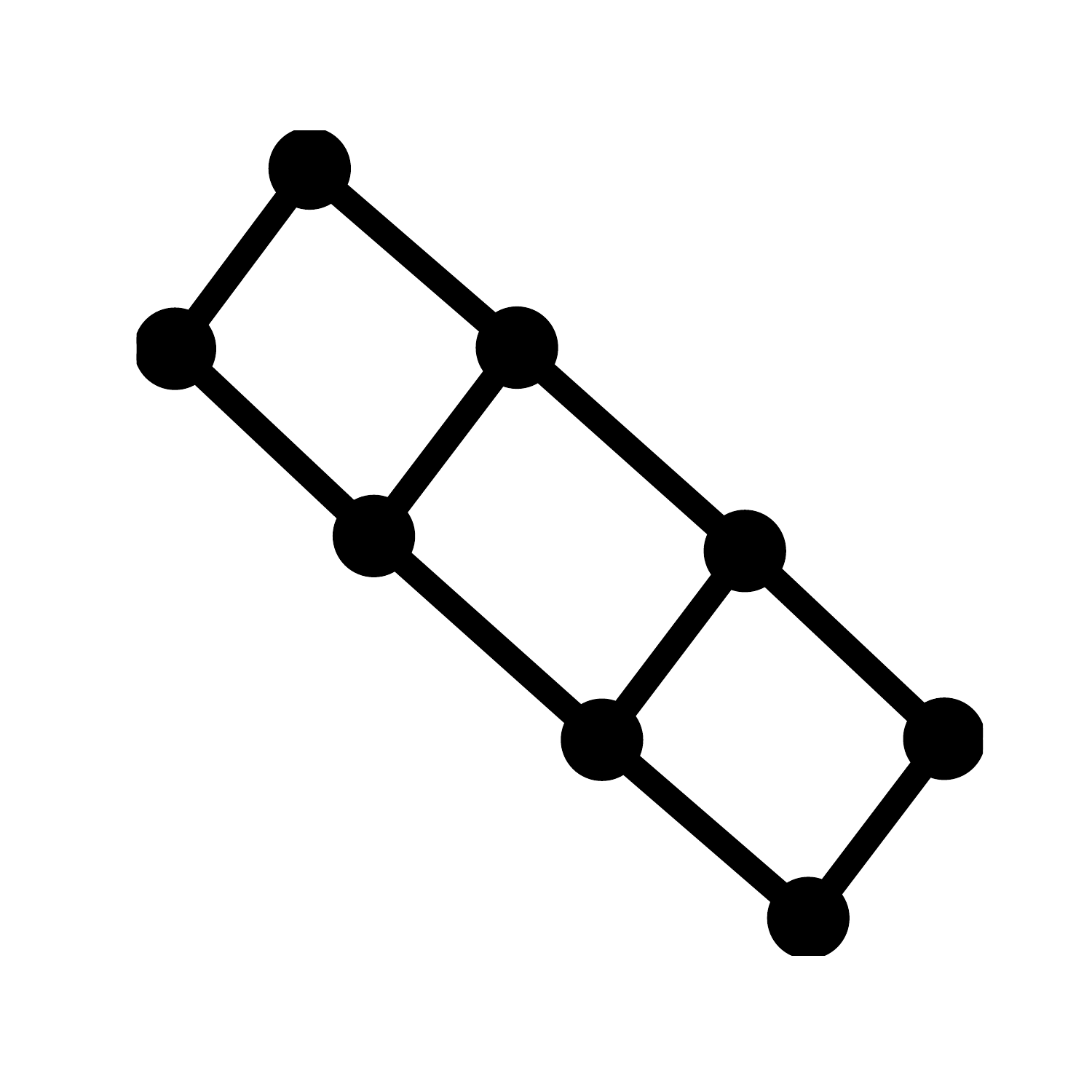}}
    \caption{Illustration of the 4 types of graphs in synthetic dataset \textit{Structure}.}
    \label{fig:illus_structure}
\end{figure}

\begin{table}[ht]
  \centering
  \begin{tabular}{ccccccc}
  \toprule
  Model & Cycle & Wheel & Path & Ladder & Avg \\
  \midrule
  RW Kernel & 62.6 & \textbf{100.0} & 74.3 & \textbf{100.0} & 84.8 \\
  WL Kernel & 55.2 & 98.5 & 65.5 & 98.5 & 79.7 \\
  SP Kernel & 61.9 & 100.0 & 43.8 & 92.2 & 75.5 \\
  GIN & - & - & - & - & 90.0\\
  GCKN & 95.0 & \textbf{100.0} & 94.9 & \textbf{100.0} & 97.5  \\
  GSKN & \textbf{99.5} & \textbf{100.0} & \textbf{99.5} & \textbf{100.0} & \textbf{99.7}  \\
  \bottomrule
    \end{tabular}
    \caption{Graph classification results on synthetic dataset \textit{Structure}.  GIN uses supervised training, leaving accuracy on each type of graphs not comparable.}
  \label{tab:sys_result1}
\end{table}

\begin{table}[ht]
  \centering
  \begin{tabular}{cccc}
  \toprule
  Model & Accuracy on \textit{Regular} \\
  \midrule
  GCKN & 50.0   \\
  GSKN & \textbf{100.0} \\
  \bottomrule
    \end{tabular}
    \caption{Graph classification results on synthetic dataset \textit{Regular}.}
  \label{tab:sys_result2}
\end{table}

\begin{table*}[t]
  \centering
  \begin{tabular}{cccccccc}
  \toprule
  Dataset & MUTAG & PROTEINS & PTC & IMDB-B & IMDB-M & COLLAB\\
 \midrule
  LDP$ ^*$ & 88.9 $\pm$ 9.6 & 73.3 $\pm$ 5.7 & 63.8 $\pm$ 6.6 & 68.5 $\pm$ 4.0 & 42.9 $\pm$ 3.7 & 76.1 $\pm$ 1.4\\
   \midrule
  GCN$ ^*$ & 85.6 $\pm$ 5.8 & 76.0 $\pm$ 3.2 & 64.2 $\pm$ 4.3 & 74.0 $\pm$ 3.4 & 51.9 $\pm$ 3.8 & 79.0 $\pm$ 1.8\\
  PatchySAN$ ^*$ & 92.6 $\pm$ 4.2 & 75.9 $\pm$ 2.8 & 60.0 $\pm$ 4.8 & 71.0 $\pm$ 2.2 & 45.2 $\pm$ 2.8 & 72.6 $\pm$ 2.2 \\
  GIN$ ^*$ & 89.4 $\pm$ 5.6 & 76.2 $\pm$ 2.8 & 64.6 $\pm$ 7.0 & 75.1 $\pm$ 5.1 & 52.3 $\pm$ 2.8 & 80.2 $\pm$ 1.9\\
  Graph2Vec & 83.1 $\pm$ 6.5 & 73.9 $\pm$ 2.6 & 59.6 $\pm$ 5.9 & 51.1 $\pm$  4.0 & 47.4 $\pm$ 3.1 & 52.6 $\pm$ 1.9 \\
  \midrule
  WL subtree$ ^*$ & 90.4 $\pm$ 5.7 & 75.0 $\pm$ 3.1 & 59.9 $\pm$ 4.3 & 73.8 $\pm$ 3.9 & 50.9 $\pm$ 3.8 & 78.9 $\pm$ 1.9 \\
  RetGK$ ^*$ & 90.3 $\pm$ 1.1 & 75.8 $\pm$ 0.6 & 62.5 $\pm$ 1.6 & 71.9 $\pm$ 1.0 & 47.7 $\pm$ 0.3 & 81.0 $\pm$ 0.3\\
  GNTK$ ^*$ & 90.0 $\pm$ 8.5 & 75.6 $\pm$ 4.2 & 67.9 $\pm$ 6.9 & 76.9 $\pm$ 3.6 & 52.8 $\pm$ 4.6 & \textbf{83.6} $\pm$ 1.0 \\
  DDGK & 85.0 $\pm$ 1.1 & 72.3 $\pm$ 4.7 & 65.7 $\pm$ 6.3 & 63.2 $\pm$ 4.5 & 49.1 $\pm$ 2.9 & 75.4 $\pm$ 1.8\\
  \midrule
  AWE & 88.9 $\pm$ 6.4 & 71.5 $\pm$ 4.9 & 63.8 $\pm$ 8.5 & 73.1 $\pm$ 4.1 & 53.1 $\pm$ 3.9 & 73.6 $\pm$ 2.3\\
  GraphSTONE & 67.2 $\pm$ 8.9 & 59.1 $\pm$ 4.4 & 55.6 $\pm$ 5.9 & 57.2 $\pm$ 2.7 & 35.9 $\pm$ 3.2 & 55.7 $\pm$ 2.4 \\
  \midrule
  GCKN$ ^*$ & 92.8$\pm$ 6.1 & 76.0 $\pm$ 3.4 & 67.3 $\pm$ 5.0 & 75.9 $\pm$ 3.7 & 53.0 $\pm$ 3.1 & 82.3 $\pm$ 1.1 \\
  GSKN & \textbf{93.3} $\pm$ 3.3 & \textbf{82.3} $\pm$  2.4 & \textbf{85.2} $\pm$ 5.1 & \textbf{79.9} $\pm$ 3.3 & \textbf{59.3} $\pm$ 3.1 & 81.8 $\pm$ 1.1 \\
  \bottomrule
    \end{tabular}
    \caption{Graph classification results on datasets without continuous node attributes. Results with * are taken from \cite{du2019graph}.}
  \label{tab:graph_classification}
\end{table*}

\begin{table*}[t]
\centering
\begin{adjustbox}{max width=0.95\linewidth}
\begin{tabular}{c c c c c c c c c c c c c c c}
\toprule
\multirow{4}{*}{Model} &
\multicolumn{4}{c}{Cora} &
\multicolumn{4}{c}{Pubmed} & 
\multicolumn{4}{c}{PPI}&
\\
\cmidrule(lr){2-5} 
\cmidrule(lr){6-9} 
\cmidrule(lr){10-13}
& 
\multicolumn{2}{c}{Macro-f1}& 
\multicolumn{2}{c}{Micro-f1}& 
\multicolumn{2}{c}{Macro-f1}& 
\multicolumn{2}{c}{Micro-f1}& 
\multicolumn{2}{c}{Macro-f1}& 
\multicolumn{2}{c}{Micro-f1}& 
\\
\cmidrule(lr){2-3} 
\cmidrule(lr){4-5}
\cmidrule(lr){6-7}
\cmidrule(lr){8-9}
\cmidrule(lr){10-11}
\cmidrule(lr){12-13}
& 
30\% & 70\% & 30\% & 70\% & 30\% & 70\% & 30\% & 70\% & 30\% & 70\% & 30\% & 70\% 
\\
\midrule
GCN$ ^*$ & 79.84 & 81.09 & 80.97 & 81.94 & 76.93 & 77.21 & 76.42 & 77.49 & 12.57 & 12.62 & 40.40 & 40.44  \\
GAT$ ^*$ & 79.33 & 82.08 & 80.41 & 83.43 & 76.94& 76.92  & 77.64& 77.82 &  11.91 & 11.97 & 39.92 & 40.10 \\
GraphSAGE$ ^*$ & 80.52 & 81.90 & 82.13 & 83.17 & 76.61 & 77.24  & 77.36 & 77.84 & 11.81 & 12.41 & 39.80 &40.08  \\
GraLSP$ ^*$ & 82.43 & 83.27 & 83.67 & 84.31 & \textbf{81.21} & 81.38  & \textbf{81.43} & 81.52 & 11.34 & 11.89 & 39.55 & 39.80 \\
GraphSTONE$ ^*$ & \textbf{82.78}& \textbf{83.54} & \textbf{83.88}& \textbf{84.73} & 78.61 & 78.87 & 79.53 & 81.03 & 15.55 & 15.91 & 43.60 & 43.64
\\
GCKN & 77.35 & 78.74 & 78.54 & 79.52 & 79.33 & 81.03 & 79.38 & 81.10 & 15.19 & 16.15 & 40.79 & 41.07 \\
GSKN & 77.44 & 79.32 & 78.87 & 80.55 & 80.06 & \textbf{81.98} & 80.26 & \textbf{82.13} & \textbf{19.19} & \textbf{20.06} & \textbf{43.71} & \textbf{43.82} \\
\midrule
Std (GSKN) / \% & 0.003 & 0.017 & 0.004 & 0.017  & 0.002 & 0.005 & 0.003 & 0.005 & 0.001 & 0.002 & 0.001 & 0.002 \\
\bottomrule
\end{tabular}
\end{adjustbox}
\caption{Macro-f1 and Micro-f1 scores of node classification. Results with * are taken from \cite{long2020graph}}
  \label{tab:node_classification}
\end{table*}

\begin{table}[t]
  \centering
  \begin{tabular}{cccc}
  \toprule
  Model &  BZR & COX2 & PROTEINS\_full  \\
  \midrule
%   RBF-WL$ ^*$ & 81.0 $\pm$ 1.7 & 75.5 $\pm$ 1.5 & 75.4 $\pm$ 0.3 \\
  HGK-WL$ ^*$ & 78.6 $\pm$ 0.6 & 78.1 $\pm$ 0.5 & 75.9 $\pm$ 0.2 \\
  HGK-SP$ ^*$ & 76.4 $\pm$ 0.7 & 72.6 $\pm$ 1.2 & 75.8 $\pm$ 0.2 \\
  WWL$ ^*$ & 84.4 $\pm$ 2.0 & 78.3 $\pm$ 0.5 & 77.9 $\pm$ 0.8\\
  GNTK$ ^*$ & 85.5 $\pm$ 0.8 & 79.6 $\pm$ 0.4 & 75.7 $\pm$ 0.2 \\ 
  \midrule
  GCKN$ ^*$ & \textbf{85.9} $\pm$ 0.5 & 81.2 $\pm$ 0.8 & 76.3 $\pm$ 0.5 \\
  GSKN & 85.3 $\pm$ 0.4 & \textbf{82.3} $\pm$ 0.6 & \textbf{81.9} $\pm$ 0.3 \\
  \bottomrule
    \end{tabular}
    \caption{Graph classification results on datasets with continuous node attributes. Results with * are taken from \cite{du2019graph}.}
  \label{tab:graph_classification2}
\end{table}

To evaluate the expressiveness of GSKN against existing GKs and GNNs in a straightforward manner, we carry out experiments on two synthetic datasets with distinct structures as a simple proof-of-concept. As mentioned in Section \ref{sec:theory}, our model is theoretically stronger than the 1-WL test, and hence all MPGNNs. 

In the first dataset \textit{Structure}, we generate graphs with 4 different structures, including Cycle, Path, wheel, and Ladder. We show illustrations of these structures in Fig. \ref{fig:illus_structure}. We generate graph instances $g$ for each class with their sizes $n \sim\mathrm{Uniform}(20, 80)$. We also randomly add 10\% edges in the generated graphs $g$ to incorporate some noise. We generate 100 graphs for each structure. 

In the other synthetic dataset \textit{Regular}, we generate 100 5-regular graphs with two types, one being connected 5-regular graphs with 20 nodes, \textit{Reg-20-5}, and the other being two disjoint 5-regular graphs with 10 nodes each, \textit{2-Reg-10-5}. As stated in \cite{sato2020survey}, the 1-WL test, and hence MPGNNs cannot distinguish among different $d$-regular graphs. 

We compare our model with classical graph kernel methods on \textit{Structure}, including the RW kernel, the WL kernel, and the SP kernel to show the general ability to distinguish among graph structures. We also compare against GNNs, including GIN and GCKN. For the WL kernel, we use degrees as node attributes following \cite{xu2018powerful}. Results are shown in Table \ref{tab:sys_result1}. It can be shown that although all methods are almost 100\% correct on Wheels and Ladders, the performances differ a lot on Cycle and Path. Among all the comparison methods, GSKN is able to outperform GCKN, the second strongest baseline by 5\% on Cycle and Path, showing its better ability in distinguishing among different structures in general. 

We further show that GSKN complements GNNs using \textit{Regular}, a dataset where ordinary GNNs would be theoretically incapable. We show the results in Table \ref{tab:sys_result2}. It can be shown that GCKN can do no better than guessing, while GSKN achieves 100\% accuracy, showing the improvement of GSKN over GNNs.

\subsection{Quantitative Evaluations}
In this section, we present results on graph and node classification. 
\subsubsection{Graph Classification}
We follow the same experimental protocols as \cite{du2019graph,xu2018powerful}, and report the average accuracy and standard deviation over a 10-fold cross-validation on each dataset. We use the same data splits as \cite{xu2018powerful} using their code. For classification, we use the SVM implementation of the Cyanure toolbox \footnote{http://julien.mairal.org/cyanure/}.

We show the graph classification results in Table \ref{tab:graph_classification} and \ref{tab:graph_classification2}. Results marked with * are taken from \cite{du2019graph}. We make the following findings. 
\begin{itemize}
    \item GSKN outperforms all baselines in all datasets except COLLAB. Specifically, GSKN achieves 6\% improvement on PROTEINS and IMDB-M, and over 10\% improvement on PTC. These improvements demonstrate the strength of GSKN in terms of graph classification. 
    \item Compared with GCKN, our model is also able to achieve improvements on all datasets except COLLAB. Since GSKN extends GCKN, the improvements endorse the effectiveness of the AWGK, and its compatibility with GCKN. 
    \item Compared with AWE and GraphSTONE, both of which use anonymous walks for graph representation, GSKN is also able to perform better. Specifically, AWE does not leverage node attributes, which limits its performance. GraphSTONE requires learning a topic model over AWs, which may require large graphs to ensure its effectiveness. By comparison, GSKN alleviates both drawbacks and achieves improvements with adequate modeling of anonymous walks, i.e. through the AWGK and its kernel mapping. 
\end{itemize}

We also show the classification results on graphs with continuous node attributes in Table \ref{tab:graph_classification2}. On these datasets, GSKN also performs competitively (similar to GCKN and GNTK on BZR, 1\% and 4\% over the best baseline on COX2 and PROTEINS\_full). We conclude that on a wide range of graph datasets can GSKN perform favorably against competitive baselines. 

\subsubsection{Node Classification}
We then carry out node classification experiments. We follow the same experimental settings as \cite{long2020graph}. Different fractions of nodes are sampled randomly for testing, leaving the rest for training. We use Logistic Regression as the classifier and take the macro and micro F1 scores for evaluation. The results are averaged over 10 independent runs.

The results are shown in Table \ref{tab:node_classification}. We make the following findings. 
\begin{itemize}
    \item GSKN achieves a 3-4\% improvement on PPI over the best baseline (GraphSTONE), a similar performance as the best baseline (GraLSP) on Pubmed, and a 1\% improvement over GCKN on Cora. Note that GCKN and GSKN do not use node-level supervision (e.g. node labels, edges), which may explain why both of them fail on Cora. 
    \item Compared with GCKN, GSKN achieves a 1\% improvement on Cora, a 1\% improvement on Pubmed, and a 3-4\% improvement on PPI. Since GCKN and GSKN are similarly trained, such improvements endorse the effectiveness of the kernel mapping of AWGK, and hence GSKN. As the node features are richer on Cora (1433 dims) and Pubmed (500 dims) than on PPI (50 dims), the additional structural information plays more significant roles on PPI than on Pubmed and Cora. Consequently, the improvements of GSKN are more evident on PPI than on Cora and Pubmed. 
\end{itemize}

\subsection{Model Analysis}
We carry out model analysis, including ablation studies, efficiency, and parameter analysis to provide better understandings of GSKN. We use graph classification to reflect the performances of the model.

\subsubsection{Ablation Studies}
Corresponding to the theoretical analysis in Section \ref{sec:theory}, we carry out ablation studies to illustrate the strength of the AWGK kernel mapping $\Psi_{AW}^{(l)}$ by modifying input features and switching model components. We use different variants of our model, and run experiments on MUTAG, PROTEINS, and PTC. 

We use four variants of our model. 
\begin{itemize}
    \item RW (degree) denotes the $\Psi_{walk, L}^{(l)}$ part with node degrees as input attributes. For graph classification, using node degrees as input attributes is a common practice, such as \cite{xu2018powerful}. 
    \item AW denotes the $\Psi_{AW, L}^{(l)}$ part only. 
    \item RW (attributes) denotes the $\Psi_{walk, L}^{(l)}$ part with node attributes as inputs. 
    \item RW (attributes) + AW denotes $\left[\Psi_{walk, L}^{(l)}\LARGE\|\Psi_{AW, L}^{(l)}\right]$, which is our full model formulation. 
\end{itemize}
The results are shown in Table \ref{tab:analysis_theory}. It can be shown that $\Psi_{AW, L}^{(l)}$ alone is a competitive graph classifier, outperforming RW (attribute). Moreover, by adding AW to RW (attribute), performances are improved by 2\% on MUTAG, 8\% on PROTEINS, and $>$ 20\% on PTC, showing the strength and compatibility of $\Psi_{AW, L}^{(l)}$ to $\Psi_{walk, L}^{(l)}$.

\subsubsection{Efficiency}
We compare the efficiency of GSKN with several GK baselines: the RW Kernel, the WL Kernel, DDGK, and AWE. We also compare the efficiency of GSKN against GCKN. We train all models and report the time needed for convergence on a single machine equipped with a GPU with 12GB RAM. 

Results are presented in Fig.\ref{tab:efficiency}. It can be shown that GSKN is comparable with the most efficient graph kernels, e.g. the WL Kernel and Graph2Vec, and takes less than three times the time of GCKN. The results demonstrate that our model is highly efficient and thus scalable. By contrast, although AWE takes similar approaches (i.e. anonymous walks) as GSKN, GSKN is over 200 times more efficient than AWE, while being more accurate as shown in Table \ref{tab:graph_classification}. 

\begin{table}[t]
  \centering
  \begin{tabular}{cccccc}
  \toprule
  Model & MUTAG & PROTEINS & PTC \\
  \midrule
  RW (degree) & 86.7 & 71.8 & 58.5 \\
  AW & 93.3 & 76.5 & 81.7 \\
  RW (attribute)  & 91.1 & 74.5 & 61.5 \\
  RW (attribute) + AW & 93.3 & 82.3 & 85.2 \\
  \bottomrule
    \end{tabular}
    \caption{Graph classification results with different model components and input features.}
  \label{tab:analysis_theory}
\end{table}

\begin{table}[t]
  \centering
  \begin{tabular}{cccc}
  \toprule
  Model &  MUTAG & IMDB-B & COLLAB \\
  \midrule
  RW Kernel & 0.9 & 55.8 & > 48h \\
  WL Kernel & 0.3 & 1.2 & 4.7 \\
  Graph2Vec & 0.5 & 1.5 & 5.4 \\
  \midrule
  DDGK & 70.2 & 360.1 & 2400.7 \\
  AWE & 348.6 & 1824.5 & > 48h \\
  GCKN & 1.1 & 1.3 & 6.5 \\
  \textbf{GSKN} & 1.1 & 1.9 & 15.3\\
  \bottomrule
    \end{tabular}
    \caption{Running time in minutes on different datasets.}
  \label{tab:efficiency}
\end{table}

\begin{figure*}[ht]
\centering
         \subfigure[Anonymous walk length $l$]{ \includegraphics[width=0.23\linewidth]{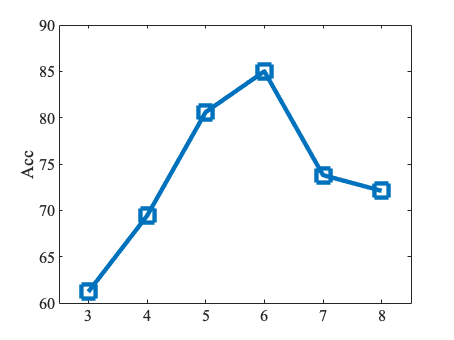}}
         \subfigure[Number of anonymous walks per node $|\Phi^l(G, u)|$]{ \includegraphics[width=0.23\linewidth]{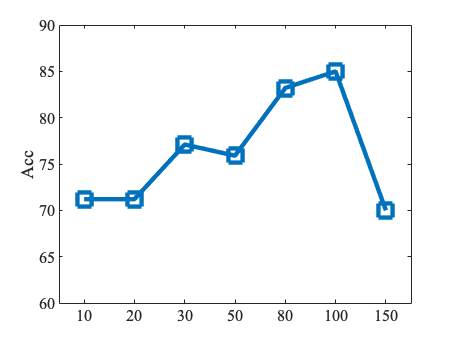}}
         \subfigure[$\alpha$ in Eqn. \ref{eqn:gaussian}]{ \includegraphics[width=0.23\linewidth]{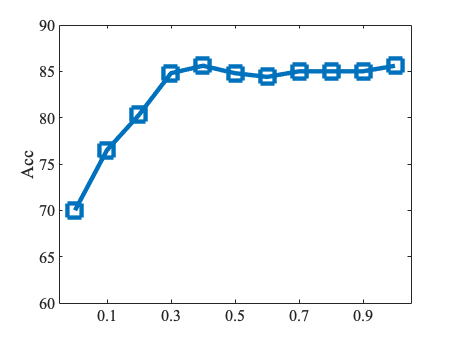}}
         \subfigure[Embedding size]{ \includegraphics[width=0.23\linewidth]{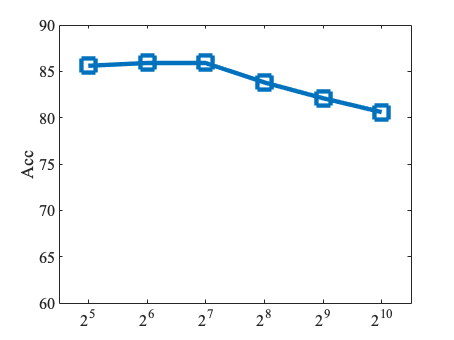}}
\caption{Parameter analysis of GSKN on PTC. }
\label{fig:para_ptc}
\end{figure*}
\begin{figure*}[ht]
\centering
         \subfigure[Anonymous walk length $l$]{ \includegraphics[width=0.23\linewidth]{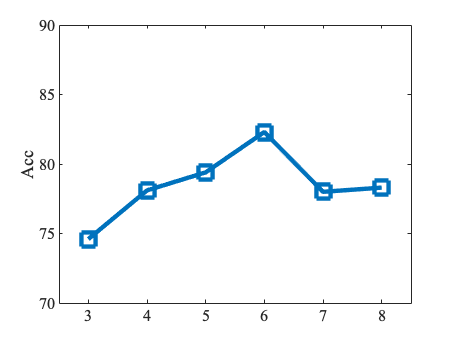}}
         \subfigure[Number of anonymous walks per node $|\Phi^l(G, u)|$]{ \includegraphics[width=0.23\linewidth]{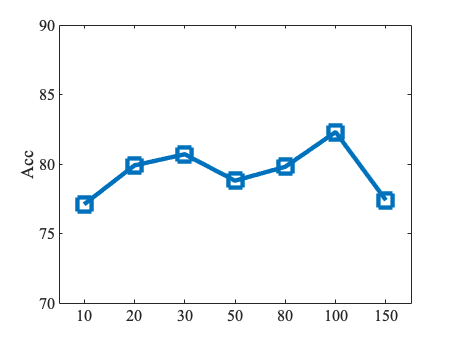}}
         \subfigure[$\alpha$ in Eqn. \ref{eqn:gaussian}]{ \includegraphics[width=0.23\linewidth]{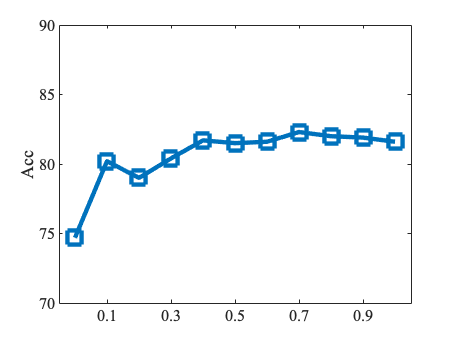}}
         \subfigure[Embedding size]{ \includegraphics[width=0.23\linewidth]{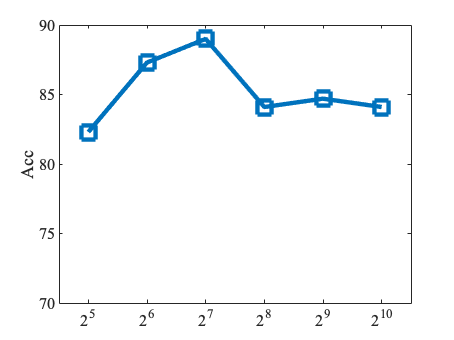}}
\caption{Parameter analysis of GSKN on PROTEINS.}
\label{fig:para_proteins}
\end{figure*}

\subsubsection{Parameter Analysis}
We analyze four parameters in our model, anonymous walk length $l$, the number of anonymous walks sampled per node $|\Phi^{l}(G, u)|$, the bandwidth parameter $\alpha$ of the Gaussian kernel in Eqn. \ref{eqn:gaussian}, and the embedding size. We conduct experiments on PTC and PROTEINS, whose results are shown in Fig. \ref{fig:para_ptc} and \ref{fig:para_proteins} respectively. We make the following findings: 
\begin{itemize}
    \item Both parameters related to the $\Psi_{AW, L}^{(l)}$, i.e. $l$ and $|\Phi^l(G, u)|$ need to be chosen appropriately instead of choosing arbitrarily large values. Specifically, $l=6$ and $|\Psi^l(G, u)| = 30$ would be most appropriate, while excessively large values may compromise the performance by as much as 10\%. 
    \item Larger $\alpha$s in the Gaussian kernel (Eqn. \ref{eqn:gaussian}) lead to better performance. Similar results are found in \cite{chen2020convolutional}, where they attribute the results to the fact that larger $\alpha$ leads to a closer resemblance between the Gaussian kernel and the Dirac kernel. 
    \item Larger embedding size does not necessarily lead to better performance. For GSKN, a size of 64 is appropriate.
\end{itemize}

\section{Conclusion}
In this paper, we propose GSKN, a GNN model with theoretically stronger ability than MPGNNs to distinguish graph structures. Specifically, we design our model based on anonymous walks (AWs), flexible substructure units, and derive it upon graph kernels (GKs), with efficient Nystr{\"o}m methods for computation. We theoretically demonstrate the stronger ability of GSKN to distinguish graph structures from both graph-level and node-level viewpoints. Correspondingly, both graph and node classification experiments are leveraged to evaluate our model, where our model outperforms a wide range of baselines, endorsing the analysis. 

\begin{acks}
This work was supported by the National Natural Science Foundation of China (Grant No. 61876006 and No. 61572041).
\end{acks}

\bibliographystyle{ACM-Reference-Format}
\bibliography{www21}
\end{document}